\documentclass{article}

\usepackage{microtype}
\usepackage{graphicx}
\usepackage{subfigure}
\usepackage{booktabs} 

\usepackage{hyperref}

\usepackage{microtype}      
\usepackage{amsmath}
\usepackage{amssymb}
\usepackage{algorithm}
\usepackage{algorithmic}
\usepackage{amsthm}

\newtheorem{assumption}{Assumption}[section]
\newtheorem{definition}{Definition}[section]
\newtheorem{theorem}{Theorem}[section]

\newtheorem{problem}{Problem}

\usepackage[accepted]{icml2021}

\icmltitlerunning{Submission and Formatting Instructions for ICML 2021}

\begin{document}

\twocolumn[
\icmltitle{Heterogeneous Risk Minimization}



\begin{icmlauthorlist}
\icmlauthor{Jiashuo Liu}{cst}
\icmlauthor{Zheyuan Hu}{cst}
\icmlauthor{Peng Cui}{cst}
\icmlauthor{Bo Li}{sem}
\icmlauthor{Zheyan Shen}{cst}
\end{icmlauthorlist}

\icmlaffiliation{cst}{Department of Computer Science and Technology, Tsinghua University, Beijing, China; Email: \{liujiashuo77, zyhu2001\}@gmail.com, cuip@tsinghua.edu.cn, shenzy17@mails.tsinghua.edu.cn.}
\icmlaffiliation{sem}{School of Economics and Management, Tsinghua University, Beijing, China; Email: libo@sem.tsinghua.edu.cn}
\icmlcorrespondingauthor{Peng Cui}{cuip@tsinghua.edu.cn}

\icmlkeywords{Machine Learning, ICML}

\vskip 0.3in
]



\printAffiliationsAndNotice{} 

\begin{abstract}
Machine learning algorithms with empirical risk minimization usually suffer from poor generalization performance due to the greedy exploitation of correlations among the training data, which are not stable under distributional shifts. 
Recently, some invariant learning methods for out-of-distribution (OOD) generalization have been proposed by leveraging multiple training environments to find invariant relationships. 
However, modern datasets are frequently assembled by merging data from multiple sources without explicit source labels. 
The resultant unobserved heterogeneity renders many invariant learning methods inapplicable. 
In this paper, we propose Heterogeneous Risk Minimization (HRM) framework to achieve joint learning of latent heterogeneity among the data and invariant relationship, which leads to stable prediction despite distributional shifts.
We theoretically characterize the roles of the environment labels in invariant learning and justify our newly proposed HRM framework. 
Extensive experimental results validate the effectiveness of our HRM framework.

\end{abstract}
\section{Introduction}
The effectiveness of machine learning algorithms with empirical risk minimization (ERM) relies on the assumption that the testing and training data are identically drawn from the same distribution, which is known as the IID hypothesis.
However, distributional shifts between testing and training data are usually inevitable due to data selection biases or unobserved confounders that widely exist in real data.
Under such circumstances, machine learning algorithms with ERM usually suffer from poor generalization performance due to the greedy exploitation of correlations among the training data, which are not stable under distributional shifts. 
How to guarantee a machine learning algorithm with out-of-distribution (OOD) generalization ability and stable performances under distributional shifts is of paramount significance, especially in high-stake applications such as medical diagnosis, criminal justice, and financial analysis etc \cite{kukar2003transductive, berk2018fairness, rudin2018optimized}.

There are mainly two branches of methods proposed to solve the OOD generalization problem, namely distributionally robust optimization (DRO) \cite{DBLP:journals/mp/EsfahaniK18, duchi2018learning, SinhaCertifying,sagawa2019distributionally} and invariant learning \cite{arjovsky2019invariant, DBLP:journals/corr/abs-2008-01883, DBLP:rationalization}.
DRO methods aim to optimize the worst-performance over a distribution set to ensure their OOD generalization performances. 
While DRO is a powerful family of methods, it is often argued for its over-pessimism problem when the distribution set is large \cite{does,frogner2019incorporating}.
From another perspective, invariant learning methods propose to exploit the causally invariant correlations(rather than varying spurious correlations) across multiple training environments, resulting in out-of-distribution (OOD) optimal predictors.
However, the effectiveness of such methods relies heavily on the quality of training environments, and the intrinsic role of environments in invariant learning remains vague in theory.
More importantly, modern big data are frequently assembled by merging data from multiple sources without explicit source labels. 
The resultant unobserved heterogeneity renders these invariant learning methods inapplicable.

In this paper, we propose Heterogeneous Risk Minimization (HRM), an optimization framework to achieve joint learning of the latent heterogeneity among the data and the invariant predictor, which leads to better generalization ability despite distributional shifts. 
More specifically, we theoretically characterize the roles of the environment labels in invariant learning, which motivates us to design two modules in the framework corresponding to heterogeneity identification and invariant learning respectively. 
We provide theoretical justification on the mutual promotion of these two modules, which resonates the joint optimization process in a reciprocal way.
Extensive experiments in both synthetic and real-world experiments datasets demonstrate the superiority of HRM in terms of average performance, stability performance as well as worst-case performance under different settings of distributional shifts.
We summarize our contributions as following:

    {\bf 1.} We propose the novel HRM framework for OOD generalization without environment labels, in which heterogeneity identification and invariant prediction are jointly optimized.
    
    {\bf 2.} We theoretically characterize the role of environments in invariant learning from the perspective of heterogeneity, based on which we propose a novel clustering method for heterogeneity identification from heterogeneous data.
    
    {\bf 3.} We theoretically justify the mutual promotion relationship between heterogeneity identification and invariant learning, resonating the joint optimization process in HRM.
\section{Problem Formulation}
\label{sec:problem setting}
\subsection{OOD and Maximal Invariant Predictor}
Following \cite{arjovsky2019invariant, DBLP:rationalization}, we consider a dataset $D=\left\{D^e\right\}_{e\in \mathrm{supp}(\mathcal{E}_{tr})}$, which is a mixture of data $D^e=\left\{(x_i^e, y_i^e)\right\}_{i=1}^{n_e}$ collected from multiple training environments $e\in \mathrm{supp}(\mathcal{E}_{tr})$, $x_i^e\in\mathcal{X}$ and $y_i^e\in\mathcal{Y}$ are the $i$-th data and label from environment $e$ respectively and $n_e$ is number of samples in environment $e$. 
Environment labels are unavailable as in most real applications.
$\mathcal{E}_{tr}$ is a random variable on indices of training environments and $P^e$ is the distribution of data and label in environment $e$.

The goal of this work is to find a predictor $f(\cdot):\mathcal{X}\rightarrow\mathcal{Y}$ with good out-of-distribution generalization performance, which can be formalized as:
\begin{small}
\begin{equation}
\label{equ:OOD}
	\arg\min_f \max_{e\in\mathrm{supp}(\mathcal{E})} \mathcal{L}(f|e) 
\end{equation}
\end{small}
where $\mathcal{L}(f|e)=\mathbb{E}[l(f(X),Y)|e]=\mathbb{E}^e[l(f(X^e),Y^e)]$ is the risk of predictor $f$ on environment $e$, and $l(\cdot,\cdot):\mathcal{Y}\times\mathcal{Y}\rightarrow\mathbb{R}^+$ is the loss function.
$\mathcal{E}$ is the random variable on indices of all possible environments such that $\mathrm{supp}(\mathcal{E})\supset \mathrm{supp}(\mathcal{E}_{tr})$. 
Usually, for all $e\in \mathrm{supp}(\mathcal{E})\setminus \mathrm{supp}(\mathcal{E}_{tr})$, the data and label distribution $P^e(X,Y)$ can be quite different from that of training environments $\mathcal{E}_{tr}$. 
Therefore, the problem in Equation \ref{equ:OOD} is referred to as Out-of-Distribution (OOD) Generalization problem \cite{arjovsky2019invariant}. 

Without any prior knowledge or structural assumptions, it is impossible to figure out the OOD generalization problem, since one cannot characterize the unseen latent environments in $\mathrm{supp}(\mathcal{E})$. 
A commonly used assumption in invariant learning literature \cite{2015Invariant, DBLP:conf/icml/GongZLTGS16, arjovsky2019invariant,DBLP:conf/aaai/KuangX0A020, DBLP:rationalization} is as follow:
\begin{assumption}
	\label{assumption1: main}
	There exists random variable $\Phi^*(X)$ such that the following properties hold: 
	
	a. $\mathrm{Invariance\ property}$:\quad for all $e, e' \in \mathrm{supp}(\mathcal{E})$, we have  $P^e(Y|\Phi^*(X)) = P^{e'}(Y|\Phi^*(X))$ holds.
	
	b. $\mathrm{Sufficiency\ property}$: $Y = f(\Phi^*) + \epsilon,\ \epsilon \perp X$. 
\end{assumption}
This assumption indicates invariance and sufficiency for predicting the target $Y$ using $\Phi^*$, which is known as invariant covariates or representations with stable relationships with $Y$ across different environments $e \in \mathcal{E}$.
In order to acquire the invariant predictor $\Phi^*(X)$, a branch of work to find maximal invariant predictor \cite{DBLP:rationalization, DBLP:journals/corr/abs-2008-01883} has been proposed, where the invariance set and the corresponding maximal invariant predictor are defined as:
\begin{definition}
	\label{definition2:invariance set}
	The invariance set $\mathcal{I}$ with respect to $\mathcal{E}$ is defined as:
	\begin{small}
	\begin{equation}
	    \begin{aligned}
	    \mathcal{I}_{\mathcal{E}} &= \{\Phi(X): Y \perp \mathcal{E}|\Phi(X)\}\\
	            &= \{\Phi(X): H[Y|\Phi(X)]=H[Y|\Phi(X),\mathcal{E}]\}
	    \end{aligned}
	\end{equation}
	\end{small}
	where $H[\cdot]$ is the Shannon entropy of a random variable. 
	The corresponding maximal invariant predictor (MIP) of $\mathcal{I}_{\mathcal{E}}$ is defined as:
	\begin{small}
	\begin{equation}
	    S = \arg\max_{\Phi \in \mathcal{I}_{\mathcal{E}}}I(Y;\Phi)
	\end{equation}
	\end{small}
    where $I(\cdot;\cdot)$ measures Shannon mutual information between two random variables.
\end{definition}
Here we prove that the MIP $S$ can can guarantee OOD optimality, as indicated in Theorem \ref{theorem:OOD opt}.
The formal statement of Theorem \ref{theorem:OOD opt} as well as its proof can be found in appendix.
\begin{theorem}
	\label{theorem:OOD opt}
	(Informal)\quad For predictor $\Phi^*(X)$ satisfying Assumption \ref{assumption1: main}, $\Phi^*$ is the maximal invariant predictor with respect to $\mathcal{E}$ and the solution to OOD problem in equation \ref{equ:OOD} is $\mathbb{E}_Y[Y|\Phi^*] = \arg\min_f \sup_{e\in\mathrm{supp}(\mathcal{E})}\mathbb{E}[\mathcal{L}(f)|e]$.
\end{theorem}
Recently, some works suppose the availability of data from multiple environments with environment labels, wherein they can find MIP \cite{DBLP:rationalization, DBLP:journals/corr/abs-2008-01883}.
However, they rely on the underlying assumption that the invariance set $\mathcal{I}_{\mathcal{E}_{tr}}$ of $\mathcal{E}_{tr}$ is exactly the invariance set $\mathcal{I}_{\mathcal{E}}$ of all possible unseen environments $\mathcal{E}$, which cannot be guaranteed as shown in Theorem \ref{theorem:subset}.
\begin{theorem}
	\label{theorem:subset}
    $\mathcal{I}_{\mathcal{E}} \subseteq \mathcal{I}_{\mathcal{E}_{tr}}$
\end{theorem}
As shown in Theorem \ref{theorem:subset} that $\mathcal{I}_{\mathcal{E}} \subseteq \mathcal{I}_{\mathcal{E}_{tr}}$, the learned predictor is only invariant to such limited environments $\mathcal{E}_{tr}$ but is not guaranteed to be invariant with respect to all possible environments $\mathcal{E}$.

Here we give a toy example in Table \ref{tab:intuition} to illustrate this.
We consider a binary classification between cats and dogs, where each photo contains 3 features, animal feature $X_1\in\left\{\text{cat},\text{dog}\right\}$, a background feature $X_2\in\left\{\text{on grass}, \text{in water}\right\}$ and the photographer's signature feature $X_3\in\left\{\text{Irma}, \text{Eric}\right\}$. Assume all possible testing environments $\mathrm{supp}(\mathcal{E})=\left\{e_1,e_2,e_3,e_4,e_5,e_6\right\}$ and the train environment $\mathrm{supp}(\mathcal{E}_{tr})=\left\{e_5,e_6\right\}$, then $\mathcal{I}_\mathcal{E}=\left\{\Phi|\Phi=\Phi(X_1)\right\}$ while $\mathcal{I}_{\mathcal{E}_{tr}}=\left\{\Phi|\Phi=\Phi(X_1,X_2)\right\}$.
The reason is that $e_5,e_6$ only tell us $X_3$ cannot be included in the invariance set but cannot exclude $X_2$.
But if $e_5$ and $e_6$ can be further divided into $e_1,e_2$ and $e_3,e_4$ respectively, the invariance set $\mathcal{I}_{\mathcal{E}_{tr}}$ becomes $\mathcal{I}_{\mathcal{E}_{tr}} = \mathcal{I}_{\mathcal{E}} = \{\Phi(X_1)\}$. 

This example shows that the manually labeled environments may not be sufficient to achieve MIP, not to mention the cases where environment labels are not available. 
This limitation necessitates the study on how to exploit the latent intrinsic heterogeneity in training data (like $e_5$ and $e_6$ in the above example) to form more refined environments for OOD generalization.
The environments need to be subtly uncovered, in the sense of OOD generalization problem, as indicated by Theorem \ref{theorem:useless env}, not all environments are helpful to tighten the invariance set.
\begin{theorem}
    \label{theorem:useless env}
    Given set of environments $\mathrm{supp}(\hat{\mathcal{E})}$, denote the corresponding invariance set $\mathcal{I}_{\hat{\mathcal{E}}}$ and the corresponding maximal invariant predictor $\hat{\Phi}$. For one newly-added environment $e_{new}$ with distribution $P^{new}(X,Y)$, if $P^{new}(Y|\hat{\Phi}) = P^e(Y|\hat{\Phi})$ for $e \in \mathrm{supp}(\hat{\mathcal{E}})$, the invariance set constrained by $\mathrm{supp}(\hat{\mathcal{E}})\cup \{e_{new}\}$ is equal to $\mathcal{I}_{\hat{\mathcal{E}}}$.
\end{theorem}

\begin{small}
\begin{table}[t]
    \centering
    \resizebox{\linewidth}{!}{
    \begin{tabular}{c|c|c|c|c|c|c}
        \small
          & \multicolumn{3}{c|}{Class 0 (Cats)} & \multicolumn{3}{c}{Class 1 (Dogs)}\\\hline
         Index & $X_1$ & $X_2$ & $X_3$ & $X_1$ & $X_2$ & $X_3$\\\hline
         $e_1$ & Cats & Water & Irma & Dogs & Grass & Eric\\\hline
         $e_2$ & Cats & Grass & Eric & Dogs & Water & Irma\\\hline
         $e_3$ & Cats & Water & Eric & Dogs & Grass & Irma\\\hline
         $e_4$ & Cats & Grass & Irma & Dogs & Water & Eric\\\hline
         $e_5$ & \multicolumn{6}{c}{Mixture: 90\% data from $e_1$ and 10\% data from $e_2$} \\\hline
         $e_6$ & \multicolumn{6}{c}{Mixture: 90\% data from $e_3$ and 10\% data from $e_4$} \\\hline
    \end{tabular}}
    \caption{A Toy Example for the difference between $\mathcal{I}_{\mathcal{E}}$ and $\mathcal{I}_{\mathcal{E}_{tr}}$.}
    \label{tab:intuition}
    \vskip -0.1in
\end{table}
\end{small}

\subsection{Problem of Heterogeneous Risk Minimization}
Besides Assumption 2.1, we make another assumption on the existence of heterogeneity in training data as:
\begin{assumption}
    \label{assumption:heterogeneity}
    $\mathrm{Heterogeneity\ Assumption}$.\\
    For random variable pair $(X,\Phi^*)$ and $\Phi^*$ satisfying Assumption \ref{assumption1: main}, using functional representation lemma \cite{networkinformation}, there exists random variable $\Psi^*$ such that $X = X(\Phi^*, \Psi^*)$, then we assume $P^e(Y|\Psi^*)$ can arbitrary change across environments $e\in \mathrm{supp}(\mathcal{E})$.
\end{assumption}

The heterogeneity among provided environments can be evaluated by the compactness of the corresponding invariance set as $|\mathcal{I}_{\mathcal{E}}|$.
Specifically, smaller $|\mathcal{I}_{\mathcal{E}}|$ leads to higher heterogeneity, since more variant features can be excluded.
Based on the assumption, we come up with the problem of heterogeneity exploitation for OOD generalization.
\begin{problem}
Heterogeneous Risk Minimization.\\
Given heterogeneous dataset $D=\left\{D^e\right\}_{e\in \mathrm{supp}(\mathcal{E}_{latent})}$ without environment labels, the task is to generate environments $\mathcal{E}_{tr}$ with minimal $|\mathcal{I}_{\mathcal{E}_{tr}}|$ and learn invariant model under learned $\mathcal{E}_{tr}$ with good OOD performance.
\end{problem}
Theorem \ref{theorem:useless env} together with Assumption \ref{assumption:heterogeneity} indicate that, to better constrain $\mathcal{I}_{\mathcal{E}_{tr}}$, the effective way is to generate environments with varying $P(Y|\Psi^*(X))$ that can exclude variant features from $\mathcal{I}_{\mathcal{E}_{tr}}$.
Under this problem setting,  we encounter the circular dependency: first we need variant $\Psi^*$ to generate heterogeneous environments $\mathcal{E}_{tr}$; then we need $\mathcal{E}_{tr}$ to learned invariant $\Phi^*$ as well as variant $\Psi^*$. 
Furthermore, there exists positive feedback between these two steps.
When acquiring $\mathcal{E}_{tr}$ with tighter $\mathcal{I}_{\mathcal{E}_{tr}}$, more invariant predictor $\Phi(X)$ (i.e. a better approximation of MIP) can be found, which will further bring a clearer picture of variant parts, and therefore promote the generation of $\mathcal{E}_{tr}$.
With this notion, we propose our framework for Heterogeneous Risk Minimization (HRM) which leverages the mutual promotion between the two steps and conduct joint optimization.

\section{Method}
In this work, we temporarily focus on a simple but general setting, where $X=[\Phi^*, \Psi^*]^T \in \mathbb{R}^d$ in raw feature level and $\Phi^*, \Psi^*$ satisfy Assumption \ref{assumption1: main}.
Under this setting, 
Our Heterogeneous Risk Minimization (HRM) framework contains two interactive parts, the frontend $\mathcal{M}_c$ for heterogeneity identification and the backend $\mathcal{M}_p$ for invariant prediction. 
The general framework is shown in Figure \ref{fig:model}.
\begin{figure}[h]
    \centering
    \includegraphics[width=0.95\linewidth]{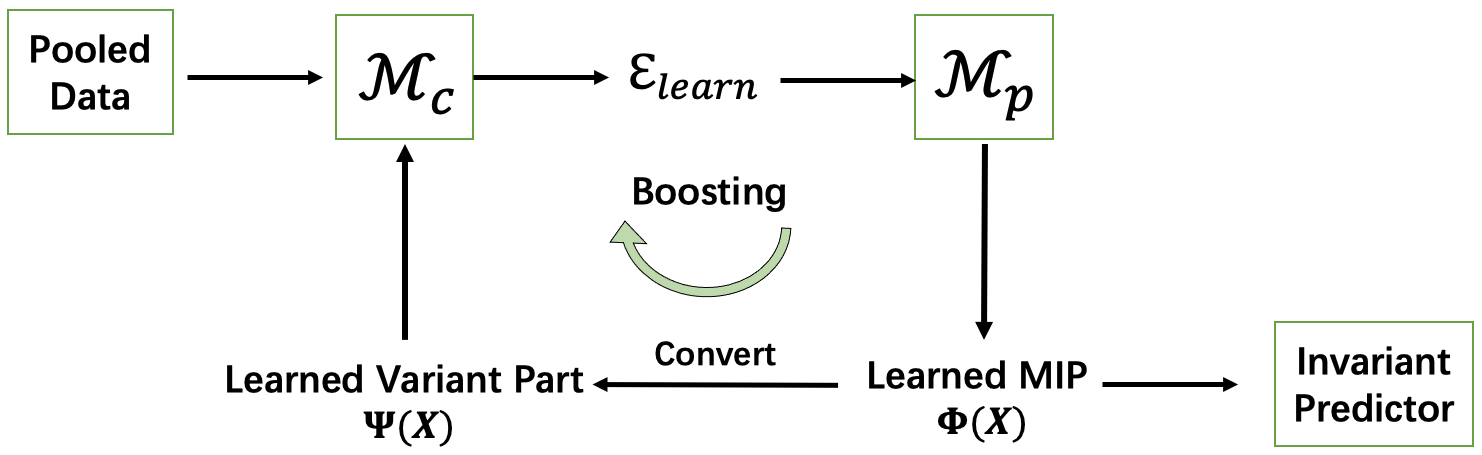}
    \caption{The framework of HRM.}
    \label{fig:model}
\end{figure}
Given the pooled heterogeneous data, it starts with the heterogeneity identification module $\mathcal{M}_c$ leveraging the learned variant representation $\Psi(X)$ to generate heterogeneous environments $\mathcal{E}_{learn}$.
Then the learned environments are used by OOD prediction module $\mathcal{M}_p$ to learn the MIP $\Phi(X)$ as well as the invariant prediction model $f(\Phi(X))$. 
After that, we derive the variant $\Psi(X)$ to further boost the module $\mathcal{M}_c$, which is supported by Theorem \ref{theorem:useless env}.
As for the 'convert' step, under our setting, we adopt feature selection in this work, through which more variant feature $\Psi$ can be attained when more invariant feature $\Phi$ is learned.
Specifically, the invariant predictor $\Phi(X)$ is generated as $\Phi(X)=M\odot X$, and the variant part $\Psi(X) = (1-M)\odot X$ correspondingly, where $M \in \left\{0,1\right\}^d$ is the binary invariant feature selection mask.
For instance, for Table \ref{tab:intuition}, $X=[X_1,X_2,X_3]$, the ground truth binary mask is $M=[1,0,0]$.  
In this way, the better $\Phi$ is learned, the better $\Psi$ can be obtained.
Note that we use the soft selection which is more flexible and general in our algorithm with $M \in [0,1]^d$.

The whole framework is jointly optimized, so that the mutual promotion between heterogeneity identification and invariant learning can be fully leveraged. 

\subsection{Implementation of $\mathcal{M}_p$}
Here we introduce our invariant prediction module $\mathcal{M}_p$, which takes multiple environments training data $D=\left\{D^e\right\}_{e\in supp(\mathcal{E}_{tr})}$ as input, and outputs the corresponding invariant predictor $f$ and the indices of invariant features $M$ given current environments $\mathcal{E}_{tr}$. 
We combine feature selection with invariant learning under heterogeneous environments, which can select the features with stable/invariant correlations with the label across $\mathcal{E}_{tr}$.
Specifically, the former module can select most informative features with respect to the loss function and latter module ensures the selected features are invariant. 
Their combination ensures $\mathcal{M}_p$ to select the most informative invariant features. 

For invariant learning, we follow the variance penalty regularizer proposed in \cite{DBLP:journals/corr/abs-2008-01883} and simplify it in feature selection scenarios.
The objective function of $\mathcal{M}_p$ with $M\in \{0,1\}^d$ is:
\begin{small}
\begin{align}
    \mathcal{L}^e(M\odot X,Y;\theta) &= \mathbb{E}_{P^e}[\ell(M\odot X^e,Y^e;\theta)]\\
    \mathcal{L}_p(M\odot X,Y;\theta) &= \mathbb{E}_{\mathcal{E}_{tr}}[\mathcal{L}^e] + \lambda \mathrm{trace}(\mathrm{Var}_{\mathcal{E}_{tr}}(\nabla_\theta\mathcal{L}^e))
\end{align}
\end{small}

However, as the optimization of hard feature selection with binary mask $M$ suffers from high variance, we use the soft feature selection with gates taking continuous value in $[0,1]$.
Specifically, following \cite{DBLP:conf/icml/YamadaLNK20}, we approximate each element of $M=[m_1,\dots,m_d]^T$ to clipped Gaussian random variable parameterized by $\mu=[\mu_1, \dots, \mu_d]^T$ as
\begin{small}
    \begin{equation}
    \label{equ:stochastic gates}
        m_i = \max\{0, \min\{1, \mu_i+\epsilon\}\}
    \end{equation}
\end{small}
where $\epsilon$ is drawn from $\mathcal{N}(0, \sigma^2)$.
With this approximation, the objective function with soft feature selection can be written as:
\begin{small}
    \begin{align}
        \mathcal{L}^e(\theta,\mu) = \mathbb{E}_{P^e}\mathbb{E}_M\left[\ell(M\odot X^e,Y^e;\theta)+\alpha \|M\|_0\right]
    \end{align}
\end{small}
where $M$ is a random vector with $d$ independent variables $m_i$ for $i \in [d]$. Under the approximation in Equation \ref{equ:stochastic gates}, $\|M\|_0$ is simply $\sum_{i \in [d]}P(m_i>0)$ and can be calculated as $\|M\|_0 = \sum_{i\in[d]} \mathrm{CDF}(\mu_i/\sigma)$, where $\mathrm{CDF}$ is the standard Gaussian CDF.
We formulate our objective as risk minimization problem:
\begin{small}
    \begin{equation}
    \label{equ:backend_obj}
        \min_{\theta,\mu}\mathcal{L}_p(\theta;\mu) = \mathbb{E}_{\mathcal{E}_{tr}}[\mathcal{L}^e(\theta,\mu)] +  \lambda \mathrm{trace}(\mathrm{Var}_{\mathcal{E}_{tr}}(\nabla_\theta\mathcal{L}^e))
    \end{equation}
\end{small}
where 
\begin{small}
    \begin{equation}
        \mathcal{L}^e(\theta,\mu) = \mathbb{E}_{P^e}\mathbb{E}_M\left[\ell(M\odot X^e,Y^e;\theta)+\alpha \sum_{i\in[d]} \mathrm{CDF}(\mu_i/\sigma) \right]
    \end{equation}
\end{small}
Further, as for linear models, we simply approximate the regularizer $\mathrm{trace}(\mathrm{Var}_{\mathcal{E}_{tr}}(\nabla_\theta\mathcal{L}^e))$ by $\|\mathrm{Var}_{\mathcal{E}_{tr}}(\nabla_\theta\mathcal{L}^e)\odot M\|^2$.
Then we obtain $\Phi(X)$ and $\Psi(X)$ when we obtain $\mu$ as well as $M$.
Further in Section \ref{sec:theorem}, we theoretically prove that the prediction module $\mathcal{M}_p$ is able to learn the MIP with respect to given environments $\mathcal{E}_{tr}$.

\subsection{Implementation of $\mathcal{M}_c$}
{\bf Notation.} $\Psi$ means the learned variant part $\Psi(X)$. $\Delta_K$ means $K$-dimension simplex. $f_\theta(\cdot)$ means the function $f$ parameterized by $\theta$.

The heterogeneity identification module $\mathcal{M}_c$ takes a single dataset as input, and outputs a multi-environment dataset partition for invariant prediction.
We implement it with a clustering algorithm.
As indicated in Theorem \ref{theorem:useless env}, the more diverse $P(Y|\Psi)$ for our generated environments, the better the invariance set $\mathcal{I}$ is.
Therefore, we cluster the data points according to the relationship between $\Psi$ and $Y$, for which we use $P(Y|\Psi)$ as the cluster centre.
Note that $\Psi$ is initialized as $\Psi(X)=X$ in our joint optimization.

Specifically, we assume the $j$-th cluster centre $P_{\Theta_j}(Y|\Psi)$ parameterized by $\Theta_j$ to be a Gaussian around $f_{\Theta_j}(\Psi)$ as $\mathcal{N}(f_{\Theta_j}(\Psi), \sigma^2)$:
\begin{small}
    \begin{equation}
	h_j(\Psi, Y) = P_{\Theta_j}(Y|\Psi) = \frac{1}{\sqrt{2\pi}\sigma} \exp(-\frac{(Y-f_{\Theta_j}(\Psi))^2}{2\sigma^2})
    \end{equation} 
\end{small}
For the given $N=\sum_{e\in\mathcal{E}_{tr}}n_e$ empirical data samples $\mathcal{D}=\{\psi_i(x_i), y_i\}_{i=1}^N$, the empirical distribution is modeled as $\hat{P}_N = \frac{1}{N}\sum_{i=1}^N \delta_i(\Psi, Y)$
where
\begin{small}
\begin{equation}
    \delta_i(\Psi,Y) = 
    \begin{cases}
       1,\quad \mathrm{if\ } \Psi = \psi_i\ \mathrm{and}\ Y = y_i\\
       0, \quad \mathrm{otherwise}
    \end{cases}
\end{equation}
\end{small}
The target of our heterogeneous clustering is to find a distribution in $\mathcal{Q} = \{Q|Q = \sum_{j\in [K]} q_j h_j(\Psi,Y), \bold{q}\in \Delta_K\}$ to fit the empirical distribution best.
Therefore, the objective function of our heterogeneous clustering is:
\begin{small}
    \begin{equation}
    \min_{Q \in \mathcal{Q}} D_{KL}(\hat{P}_N\|Q)
    \end{equation}
\end{small}
The above objective can be further simplified to:
\begin{small}
\begin{equation}
    \label{equ:clustering_final}
    \min_{\Theta, \bold{q}} \left\{\mathcal{L}_c = -\frac{1}{N}\sum_{i=1}^N\log\left[\sum_{j=1}^Kq_j h_j(\psi_i, y_i)\right]\right\}
\end{equation}
\end{small}
As for optimization, we use EM algorithm to optimize the centre parameter $\Theta$ and the mixture weight $\bold{q}$.
After optimizing equation \ref{equ:clustering_final}, for building $\mathcal{E}_{tr}$, we assign each data point to environment $e_j \in \mathcal{E}_{tr}$ with probability:
\begin{small}
    \begin{equation}
        P(e_j|\Psi,Y)=q_jh_j(\Psi,Y)/\left(\sum_{i=1}^Kq_ih_i(\Psi,Y)\right)
    \end{equation}
\end{small}
In this way, $\mathcal{E}_{tr}$ is generated by $\mathcal{M}_c$.

\section{Theoretical Analysis}
\label{sec:theorem}
In this section, we theoretically analyze our proposed Heterogeneous Risk Minimization (HRM) method. 
We first analyze our proposed invariant learning module $\mathcal{M}_p$, and then justify the existence of the positive feedback in our HRM. 

{\bf Justification of $\mathcal{M}_p$}\quad We prove that given training environments $\mathcal{E}_{tr}$, our invariant prediction model $\mathcal{M}_p$ can learn the maximal invariant predictor $\Phi(X)$ with respect to the corresponding invariance set $\mathcal{I}_{\mathcal{E}_{tr}}$.

\begin{theorem}
    Given $\mathcal{E}_{tr}$, the learned $\Phi(X)=M\odot X$ is the maximal invariant predictor of $\mathcal{I}_{\mathcal{E}_{tr}}$.
\end{theorem}

{\bf Justification of the Positive Feedback}\quad The core of our HRM framework is the mechanism for $\mathcal{M}_c$ and $\mathcal{M}_p$ to mutual promote each other.
Here we theoretically justify the existence of such positive feedback.
In Assumption \ref{assumption1: main}, we assume the invariance and sufficiency properties of the stable features $\Phi^*$ and assume the relationship between unstable part $\Psi^*$ and $Y$ can arbitrarily change. 
Here we make a more specific assumption on the heterogeneity across environments with respect to $\Phi^*$ and $\Psi^*$.
\begin{assumption}
    \label{assumption:v}
    Assume the pooled training data is made up of heterogeneous data sources: $P_{tr} = \sum_{e \in \mathrm{supp}(\mathcal{E}_{tr})} w_e P^e$. For any $e_i, e_j \in \mathcal{E}_{tr}, e_i\neq e_j$, we assume
    \begin{small}
        \begin{equation}
        \label{equ:v_assumption}
            I^c_{i,j}(Y;\Phi^*|\Psi^*) \geq \mathrm{max}(I_i(Y;\Phi^*|\Psi^*),I_j(Y;\Phi^*|\Psi^*))
        \end{equation}
    \end{small}
    where $\Phi^*$ is invariant feature and $\Psi^*$ the variant. $I_i$ represents mutual information in $P^{e_i}$ and $I^c_{i,j}$ represents the cross mutual information between $P^{e_i}$ and $P^{e_j}$ takes the form of $I^c_{i,j}(Y;\Phi|\Psi) = H^c_{i,j}[Y|\Psi] - H^c_{i,j}[Y|\Phi,\Psi]$ and $H^c_{i,j}[Y] = -\int p^{e_i}(y)\log p^{e_j}(y) dy$.
\end{assumption}
Here we would like to intuitively demonstrate this assumption.
Firstly, the mutual information $I_i(Y;\Phi^*)=H_i[Y]-H_i[Y|\Phi^*]$ can be viewed as the error reduction if we use $\Phi^*$ to predict $Y$ rather than predict by nothing.
Then the cross mutual information $I_{i,j}(Y;\Phi^*)$ can be viewed as the error reduction if we use the predictor learned on $\Phi^*$ in environment $e_j$ to predict in environment $e_i$, rather than predict by nothing.
Therefore, the R.H.S in equation \ref{equ:v_assumption} measures that, in environment $e_i$, how much prediction error can be reduced if we further add $\Phi^*$ for prediction rather than use only $\Psi^*$.
And the L.H.S measures that, when using predictors trained in $e_i$ to predict in $e_j$, how much prediction error can be reduced if we further add $\Phi^*$ for prediction rather than use only $\Psi^*$.
Intuitively, Assumption \ref{assumption:v} assumes that invariant feature $\Phi^*$ provides more information for predicting $Y$ across environments than in one single environment, and correspondingly, the information provided by $\Psi^*$ shrinks a lot across environments, which indicates that the relationship between variant feature $\Psi^*$ and $Y$ varies across environments.
Based on this assumption, we first prove that the cluster centres are pulled apart as invariant feature is excluded from clustering.
\begin{theorem}
    \label{theorem:kl}
    For $e_i, e_j \in \mathrm{supp}(\mathcal{E}_{tr})$, assume that $X = [\Phi^*, \Psi^*]^T$ satisfying Assumption \ref{assumption1: main}, where $\Phi^*$ is invariant and $\Psi^*$ variant.
    Then under Assumption \ref{assumption:v}, we have $
        \mathrm{D_{KL}}(P^{e_i}(Y|X)\|P^{e_j}(Y|X)) \leq \mathrm{D_{KL}}(P^{e_i}(Y|\Psi^*)\|P^{e_j}(Y|\Psi^*))$
\end{theorem}
Theorem \ref{theorem:kl} indicates that the distance between cluster centres is larger when using variant features $\Psi^*$ , therefore, it is more likely to obtain the desired heterogeneous environments, which explains why we use learned variant part $\Psi(X)$ for clustering.
Finally, we provide the theorem for optimality guarantee for our HRM. 
\begin{theorem}
    \label{theorem:optimal}
    Under Assumption \ref{assumption1: main} and \ref{assumption:v}, for the proposed $\mathcal{M}_c$ and $\mathcal{M}_p$, we have the following conclusions:
    1. Given environments $\mathcal{E}_{tr}$ such that $\mathcal{I}_{\mathcal{E}}=\mathcal{I}_{\mathcal{E}_{tr}}$, the learned $\Phi(X)$ by $\mathcal{M}_p$ is the maximal invariant predictor of $\mathcal{I}_{\mathcal{E}}$.
  
    2. Given the maximal invariant predictor $\Phi^*$ of $\mathcal{I}_{\mathcal{E}}$, assume the pooled training data is made up of data from all environments in $\mathrm{supp}(\mathcal{E})$, there exist one split that achieves the minimum of the objective function and meanwhile the invariance set regularized is equal to $\mathcal{I}_{\mathcal{E}}$.
\end{theorem}
Intuitively, Theorem \ref{theorem:optimal} proves that given one of the $\mathcal{M}_c$ and $\mathcal{M}_p$ optimal, the other is optimal, which validates the existence of the global optimal point of our algorithm.


\section{Experiment}
In this section, we validate the effectiveness of our method on simulation data and real-world data.

{\bf Baselines}\quad
 We compare our proposed HRM with the following methods:
 \begin{small}
     \begin{itemize}
     \item Empirical Risk Minimization(ERM): $\min_\theta \mathbb{E}_{P_0}[\ell(\theta;X,Y)]$
     \item Distributionally Robust Optimization(DRO \cite{SinhaCertifying}):\quad $\min_\theta \sup_{Q\in W(Q,P_0)\leq \rho}\mathbb{E}_Q[\ell(\theta;X,Y)]$
     \item Environment Inference for Invariant Learning(EIIL \cite{creager2020environment}):
         \begin{small}
         \begin{equation}
             \begin{aligned}
                     \min_\Phi \max_u &\sum_{e\in\mathcal{E}}\frac{1}{N_e}\sum_i u_i(e)\ell(w\odot\Phi(x_i),y_i) + \\
                     &\sum_{e\in\mathcal{E}}\lambda \|\nabla_{w|w=1.0}\frac{1}{N_e}\sum_iu_i(e)\ell(w\odot\Phi(x_i),y_i)\|_2
             \end{aligned}
         \end{equation}
         \end{small}
     \item Invariant Risk Minimization(IRM \cite{arjovsky2019invariant}) with environment $\mathcal{E}_{tr}$ labels:
         \begin{small}
         \begin{equation}
             \min_\Phi \sum_{e\in\mathcal{E}_{tr}}\mathcal{L}^e + \lambda \|\nabla_{w|w=1.0}\mathcal{L}^e(w\odot \Phi)\|^2
         \end{equation}
         \end{small}
 \end{itemize}
 \end{small}
Further, for ablation study, we also compare with HRM$^s$, which runs HRM for only one iteration without the feedback loop.
Note that IRM is based on multiple training environments and we provide environment $\mathcal{E}_{tr}$ labels for IRM, while others do not need environment labels.

{\bf Evaluation Metrics}\quad
To evaluate the prediction performance, we use $\mathrm{Mean\_Error}$ defined as $\mathrm{Mean\_Error} = \frac{1}{|\mathcal{E}_{test}|} \sum_{e\in \mathcal{E}_{test}} \mathcal{L}^e$, $\mathrm{Std\_Error}$ defined as $\mathrm{Std\_Error} = \sqrt{\frac{1}{|\mathcal{E}_{test}|-1}\sum_{e\in \mathcal{E}_{test}}(\mathcal{L}^e-\mathrm{Mean\_Error})^2}$, which are mean and standard deviation error across $\mathcal{E}_{test}$ and $\mathrm{Max\_Error}=\max_{e\in\mathcal{E}_{test}}\mathcal{L}^e$, which are mean error, standard deviation error and worst-case error across $\mathcal{E}_{test}$.

{\bf Imbalanced Mixture}\quad
It is a natural phenomena that empirical data follow a power-law distribution, i.e. only a few environments/subgroups are common and the rest are rare \cite{2018Causally, sagawa2019distributionally, 2020An}.
Therefore, we perform non-uniform sampling among different environments in training set.

\begin{table*}[t]
	\centering
	\caption{Results in selection bias simulation experiments of different methods with varying selection bias $r$, and dimensions $n_{b}$ and $d$ of training data, and each result is averaged over ten times runs.}
	\label{tab:sim_selection}
	\vskip 0.05in
	
	\resizebox{0.8\textwidth}{28mm}{
	\begin{tabular}{|l|c|c|c|c|c|c|c|c|c|}
		\hline
		\multicolumn{10}{|c|}{\textbf{Scenario 1: varying selection bias rate $r$\quad($d=10,n_b=1$)}}\\
		\hline
		$r$&\multicolumn{3}{|c|}{$r=1.5$}&\multicolumn{3}{|c|}{$r=1.9$}&\multicolumn{3}{|c|}{$r=2.3$}\\
		\hline
		Methods &  $\mathrm{Mean\_Error}$ & $\mathrm{Std\_Error}$& $\mathrm{Max\_Error}$ &$\mathrm{Mean\_Error}$ & $\mathrm{Std\_Error}$& $\mathrm{Max\_Error}$ &  $\mathrm{Mean\_Error}$ & $\mathrm{Std\_Error}$ & $\mathrm{Max\_Error}$ \\
		\hline 
		ERM & 0.476 & 0.064 & 0.524&0.510 & 0.108& 0.608  & 0.532 & 0.139& 0.690  \\
		DRO& 0.467 & 0.046& 0.516 & 0.512 & 0.111& 0.625  & 0.535 & 0.143& 0.746 \\
		EIIL & 0.477 & 0.057& 0.543 & 0.507 & 0.102& 0.613 & 0.540 & 0.139& 0.683 \\
		\hline
		IRM(with $\mathcal{E}_{tr}$ label) & 0.460 & 0.014& 0.475 & 0.456 & 0.015& 0.472 & 0.461 & 0.015& 0.475\\
		\hline
		HRM$^s$ & 0.465 & 0.045& 0.511 & 0.488 & 0.078& 0.577 & 0.506 & 0.096& 0.596 \\
		HRM &\bf 0.447 &\bf 0.011& \bf 0.462 &\bf  0.449 &\bf 0.010& \bf 0.465  &\bf 0.447 &\bf 0.011& \bf 0.463 \\
		\hline
		\multicolumn{10}{|c|}{\textbf{Scenario 2: varying dimension $d$\quad($r = 1.9, n_{b}=0.1d$)}}\\
		\hline
		$d$&\multicolumn{3}{|c|}{$d=10$}&\multicolumn{3}{|c|}{$d=20$}&\multicolumn{3}{|c|}{$d=40$}\\
		\hline
		Methods &  $\mathrm{Mean\_Error}$ & $\mathrm{Std\_Error}$& $\mathrm{Max\_Error}$ & $\mathrm{Mean\_Error}$ & $\mathrm{Std\_Error}$& $\mathrm{Max\_Error}$ &  $\mathrm{Mean\_Error}$ & $\mathrm{Std\_Error}$ & $\mathrm{Max\_Error}$ \\
		\hline %
		ERM & 0.510 & 0.108& 0.608 & 0.533 & 0.141& 0.733  & 0.528 & 0.175& 0.719 \\ 
		DRO & 0.512 & 0.111& 0.625 & 0.564 & 0.186& 0.746 & 0.555 & 0.196& 0.758 \\
		EIIL & 0.507 & 0.102& 0.613 & 0.543 & 0.147& 0.699 & 0.542	& 0.178& 0.727\\
		\hline
		IRM(with $\mathcal{E}_{tr}$ label) & 0.456 & 0.015& 0.472 & 0.484 & 0.014& 0.489 & 0.500 & 0.051& 0.540\\
		\hline
		HRM$^s$ & 0.488 & 0.078& 0.577 & 0.486 & 0.069& 0.555 & 0.477 & 0.081& 0.553 \\
		HRM &\bf 0.449 & \bf 0.010& \bf 0.465 & \bf 0.466 & \bf 0.011& \bf0.478 & \bf 0.465 & \bf 0.015& \bf 0.482\\
		\hline
	\end{tabular}
	}
\end{table*}

\subsection{Simulation Data}
We design two mechanisms to simulate the varying correlations among covariates across environments, named by selection bias and anti-causal effect.

\begin{table*}[htbp]
	\centering
	\caption{Prediction errors of the anti-causal effect experiment. We design two settings with different dimensions of $\Phi^*$ and $\Psi^*$ as $n_\phi$ and $n_\psi$ respectively. The results are averaged over 10 runs.}
	\label{tab:anti-causal}
	\vskip 0.05in
	
	\resizebox{0.75\textwidth}{30mm}{
	\begin{tabular}{|l|c|c|c|c|c|c|c|c|c|c|}
		\hline
		\multicolumn{11}{|c|}{\textbf{Scenario 1: $n_\phi=9,\ n_\psi=1$}}\\
		\hline
		$e$&\multicolumn{3}{|c|}{Training environments}&\multicolumn{7}{|c|}{Testing environments}\\
		\hline
		Methods &  $e_1$ & $e_2$ &$e_3$ & $e_4$ &  $e_5$ & $e_6$ &$e_7$ & $e_8$  & $e_9$ & $e_{10}$  \\
		\hline 
		ERM & 0.290 & 0.308 & 0.376 & 0.419 & 0.478 & 0.538& 0.596 & 0.626 & 0.640 & 0.689  \\
		DRO &0.289 & 0.310 & 0.388 & 0.428 & 0.517 & 0.610 & 0.627 & 0.669 & 0.679 & 0.739 \\
		EIIL & \bf 0.075 &\bf 0.128 & 0.349 & 0.485 & 0.795 & 1.162 & 1.286 & 1.527 & 1.558 & 1.884 \\
		\hline
		IRM(with $\mathcal{E}_{tr}$ label) & 0.306 & 0.312 & 0.325  & 0.328 & 0.343 & 0.358 & 0.365& 0.374 & 0.377 & 0.392\\
		\hline
		HRM$^s$& 1.060 & 1.085 & 1.112  & 1.130 &  1.207 & 1.280 & 1.325& 1.340 &  1.371 & 
        1.430 \\
		HRM &0.317 & 0.314 &\bf 0.322 &\bf  0.318 &\bf 0.321 & \bf 0.317 &\bf 0.315 &\bf 0.315 &\bf 0.316 &\bf 0.320 \\
		\hline
		
		\multicolumn{11}{|c|}{\textbf{Scenario 2: $n_\phi=5,\ n_\psi=5$}}\\
		\hline
		$e$&\multicolumn{3}{|c|}{Training environments}&\multicolumn{7}{|c|}{Testing environments}\\
		\hline
		Methods &  $e_1$ & $e_2$ &$e_3$ & $e_4$ &  $e_5$ & $e_6$ &$e_7$ & $e_8$  & $e_9$ & $e_{10}$  \\
		\hline 
		ERM & 0.238 &  0.286 & 0.433 & 0.512 & 0.629 & 0.727 &
  0.818 &  0.860 &  0.895 &  0.980  \\
        DRO &0.237 & 0.294 & 0.452 & 0.529 & 0.651 & 0.778 &
  0.859 &  0.911 &0.950 & 1.028 \\  
		EIIL & \bf 0.043 &\bf  0.145 & 0.521& 0.828 & 1.237 & 1.971 &  2.523 & 2.514 &  2.506 & 3.512 \\
		\hline
        IRM(with $\mathcal{E}_{tr}$ label) &0.287 & 0.293 & 0.329 & 0.345 & 0.382 & 0.420 & 0.444 &0.461 & 0.478 & 0.504\\
		\hline
		HRM$^s$&0.455 & 0.463 & 0.479& 0.478 & 0.495 &  0.508 &  0.513 &  0.519  & 0.525 & 
        0.533 \\
		HRM &0.316 & 0.315 &\bf 0.315 &\bf 0.330 &\bf 0.3200 &\bf 0.317 &\bf
  0.326 &\bf  0.330 &\bf 0.333 &\bf  0.335 \\
		\hline
	\end{tabular}
	}
\end{table*}

{\bf Selection Bias}\quad 
In this setting, the correlations between variant covariates and the target are perturbed through selection bias mechanism. 
According to Assumption \ref{assumption1: main}, we assume $X = [\Phi^*,\Psi^*]^T \in \mathbb{R}^d$ and $Y = f(\Phi^*) + \epsilon$ and that $P(Y|\Phi^*)$ remains invariant across environments while $P(Y|\Psi^*)$ changes arbitrarily. 
For simplicity, we select data points according to a certain variable set $V_b \subset \Psi^*$:
\begin{small}
\begin{equation}
\hat{P}(x) = \prod_{v_i \in V_b}|r|^{-5*|f(\phi^*) - sign(r)*v_i|}
\end{equation}  
\end{small}
where $|r| > 1$, $V_b \in \mathbb{R}^{n_b}$ and $\hat{P}(x)$ denotes the probability of point $x$ to be selected.
Intuitively, $r$ eventually controls the strengths and direction of the spurious correlation between $V_b$ and $Y$(i.e. if $r>0$, a data point whose $V_b$ is close to its $y$ is more probably to be selected.).
The larger value of $|r|$ means the stronger spurious correlation between $V_b$ and $Y$, and $r \ge 0$ means positive correlation and vice versa. 
Therefore, here we use $r$ to define different environments.

In training, we generate $sum=2000$ data points, where $\kappa=95\%$ points from environment $e_1$ with a predefined $r$ and $1-\kappa=5\%$ points from $e_2$ with $r=-1.1$. In testing, we generate data points for 10 environments with $r \in [-3, -2.7, -2.3, \dots ,2.3,2.7,3.0]$. $\beta$ is set to 1.0. 
We compare our HRM with ERM, DRO, EIIL and IRM for Linear Regression. 
We conduct extensive experiments with different settings on $r$, $n_b$ and $d$. 
In each setting, we carry out the procedure 10 times and report the average results. 
The results are shown in Table \ref{tab:sim_selection}. 

From the results, we have the following observations and analysis: {\bf ERM} suffers from the distributional shifts in testing and yields poor performance in most of the settings.
{\bf DRO} surprisingly has the worst performance, which we think is due to the over-pessimism problem \cite{frogner2019incorporating}.
{\bf EIIL} has the similar performance with ERM, which indicates that its inferred environments cannot reveal the spurious correlations between $Y$ and $V_b$.
{\bf IRM} performs much better than the above two baselines, however, as IRM depends on the available environment labels to work, it uses much more information than the other three methods.
Compared to the three baselines, our {\bf HRM} achieves nearly perfect performance with respect to average performance and stability, especially the variance of losses across environments is close to 0, which reflects the effectiveness of our heterogeneous clustering as well as the invariant learning algorithm.
Furthermore, our HRM does not need environment labels, which verifies that our clustering algorithm can mine the latent heterogeneity inside the data and further shows our superiority to IRM.

Besides, we visualize the differences between environments using Task2Vec \cite{task2vec} in Figure \ref{fig:task2vec}, where larger value means the two environments are more heterogeneous.
The pooled training data are mixture of environments with $r=1.9$ and $r=-1.1$, the difference between whom is shown in yellow box.
And the red boxes show differences between learned environments by HRM$^s$ and HRM.
The big promotion between $\mathcal{E}_{init}$ and $\mathcal{E}_{learn}$ verifies our HRM can exploit heterogeneity inside data as well as the existence of the positive feedback.
Due to space limitation, results of varying $sum,\kappa, n_b$ as well as experimental details are left to appendix.

\begin{figure}[]
    \includegraphics[width=\linewidth]{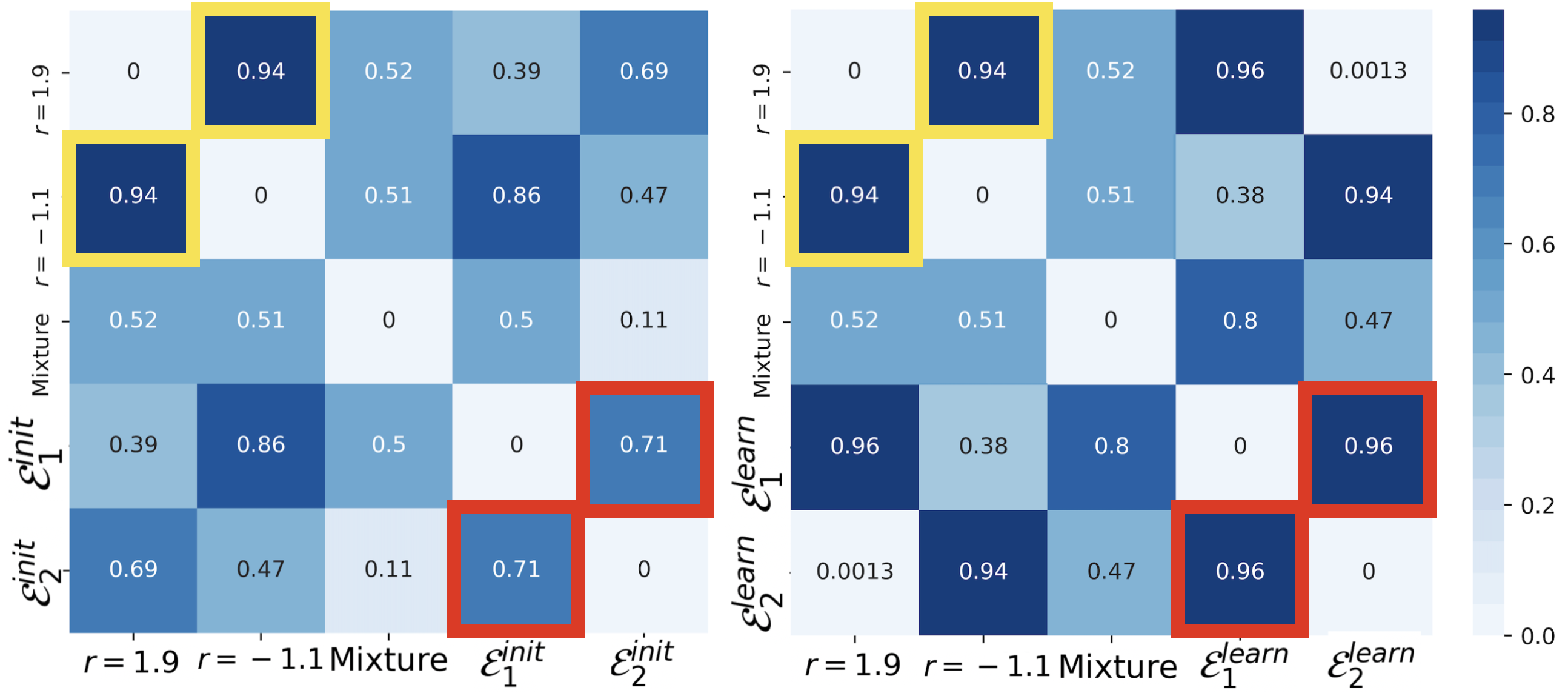}
    \vskip -0.1in
    \caption{Visualization of differences between environments in scenario 1 in selection bias experiment($r=1.9$). The left figure shows the initial clustering results using $X$, and the right one shows the learned $\mathcal{E}^{learn}$ using the learned variant part $\Psi(X)$.}
    \label{fig:task2vec}
\end{figure}
\vskip -0.05in

{\bf Anti-causal Effect}\quad 
Inspired by \cite{arjovsky2019invariant}, we induce the spurious correlation by using anti-causal relationship from the target $Y$ to the variant covariates $\Psi^*$.
In this experiment, we assume $X=[\Phi^*,\Psi^*]^T \in \mathbb{R}^d$ and firstly sample $\Phi^*$ from mixture Gaussian distribution characterized as $\sum_{i=1}^k z_k \mathcal{N}(\mu_i,I)$ and the target $Y = \theta_\phi^T\Phi^* + \beta \Phi_1\Phi_2\Phi_3+\mathcal{N}(0,0.3)$.
Then the spurious correlations between $\Psi^*$ and $Y$ are generated by anti-causal effect as 
\begin{small}
    \begin{equation}
        \Psi^* = \theta_\psi Y + \mathcal{N}(0,\sigma(\mu_i)^2)
    \end{equation}
\end{small}
where $\sigma(\mu_i)$ means the Gaussian noise added to $\Psi^*$ depends on which component the invariant covariates $\Phi^*$ belong to. 
Intuitively, in different Gaussian components, the corresponding correlations between $\Psi^*$ and $Y$ are varying due to the different value of $\sigma(\mu_i)$. 
The larger the $\sigma(\mu_i)$ is, the weaker correlation between $\Psi^*$ and $Y$. 
We use the mixture weight $Z=[z_1,\dots,z_k]^T$ to define different environments, where different mixture weights represent different overall strength of the effect $Y$ on $\Psi^*$.

In this experiment, we set $\beta=0.1$ and build 10 environments with varying $\sigma$ and the dimension of $\Phi^*,\Psi^*$, the first three for training and the last seven for testing. 
We run experiments for 10 times and the averaged results are shown in Table \ref{tab:anti-causal}.
{\bf EIIL} achieves the best training performance with respect to prediction errors on training environments $e_1$, $e_2$, $e_3$, while its performances in testing are poor.
{\bf ERM} suffers from distributional shifts in testing.
{\bf DRO} seeks for over-considered robustness and performs much worse.
{\bf IRM} performs much better as it learns invariant representations with help of environment labels.
{\bf HRM} achieves nearly uniformly good performance in training environments as well as the testing ones, which validates the effectiveness of our method and proves its excellent generalization ability.

\begin{figure*}[ht]
\vskip -0.1in
\subfigure[Training and testing accuracy for the car insurance prediction. Left sub-figure shows the training results for 5 settings and the right shows their corresponding testing results.]{
\label{img:car renting}
\includegraphics[width=0.32\linewidth]{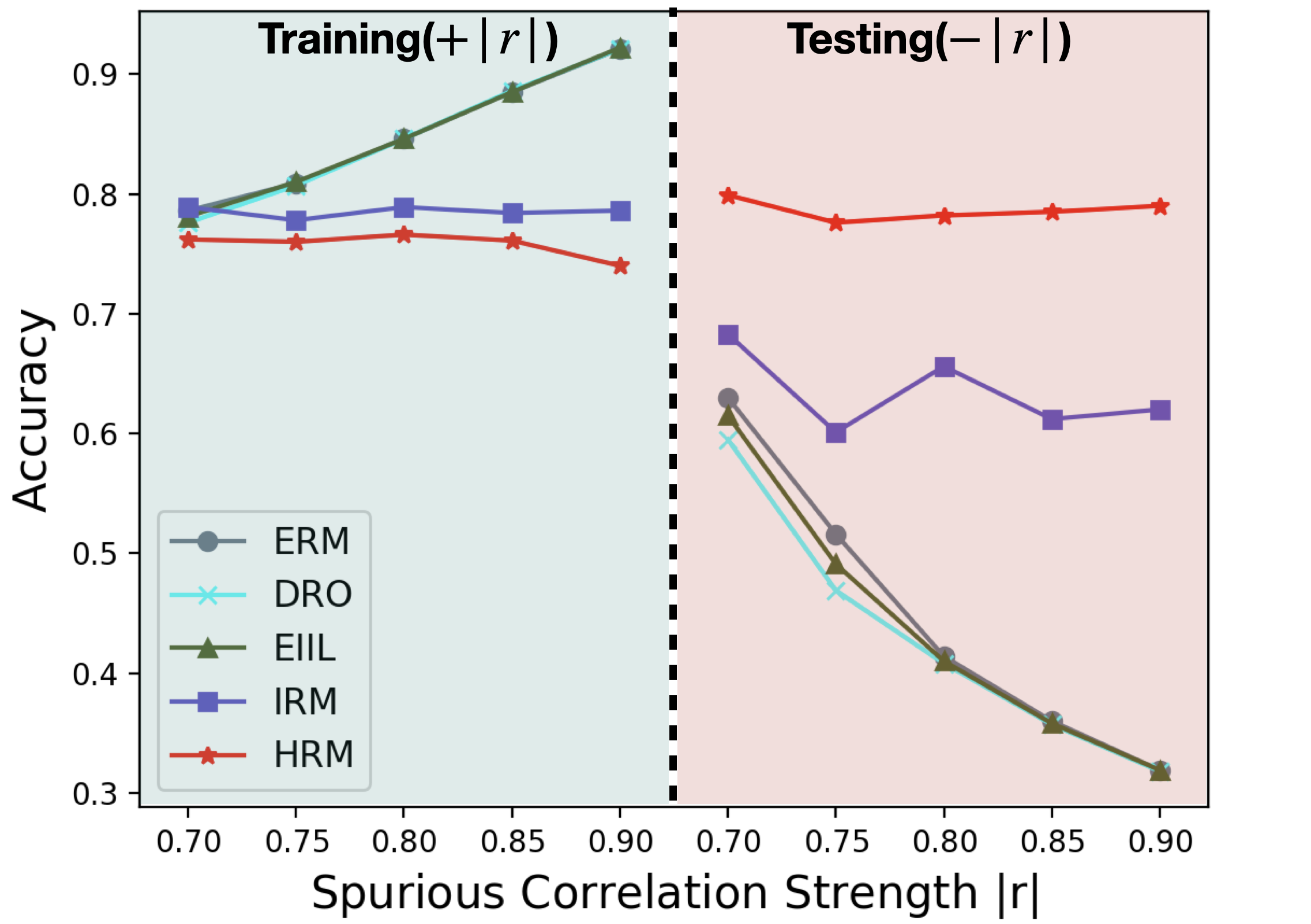}}
\subfigure[Mis-Classification Rate for the income prediction.]{
\label{img:income}
\includegraphics[width=0.32\linewidth]{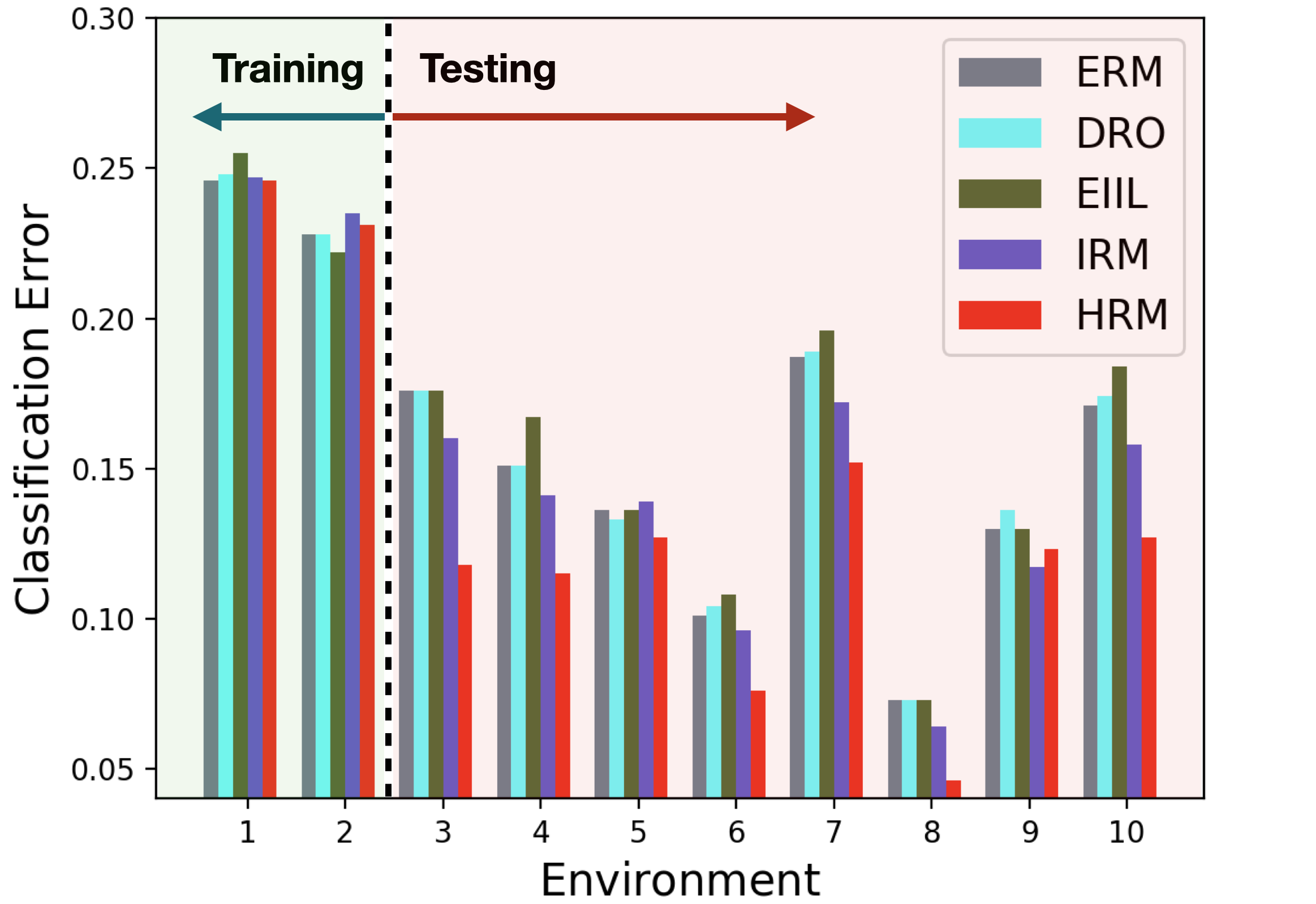}}
\subfigure[Prediction error for the house price prediction. $\mathrm{RMSE}$ refers to the Root Mean Square Error.]{
\label{img:house price}
\includegraphics[width=0.32\linewidth]{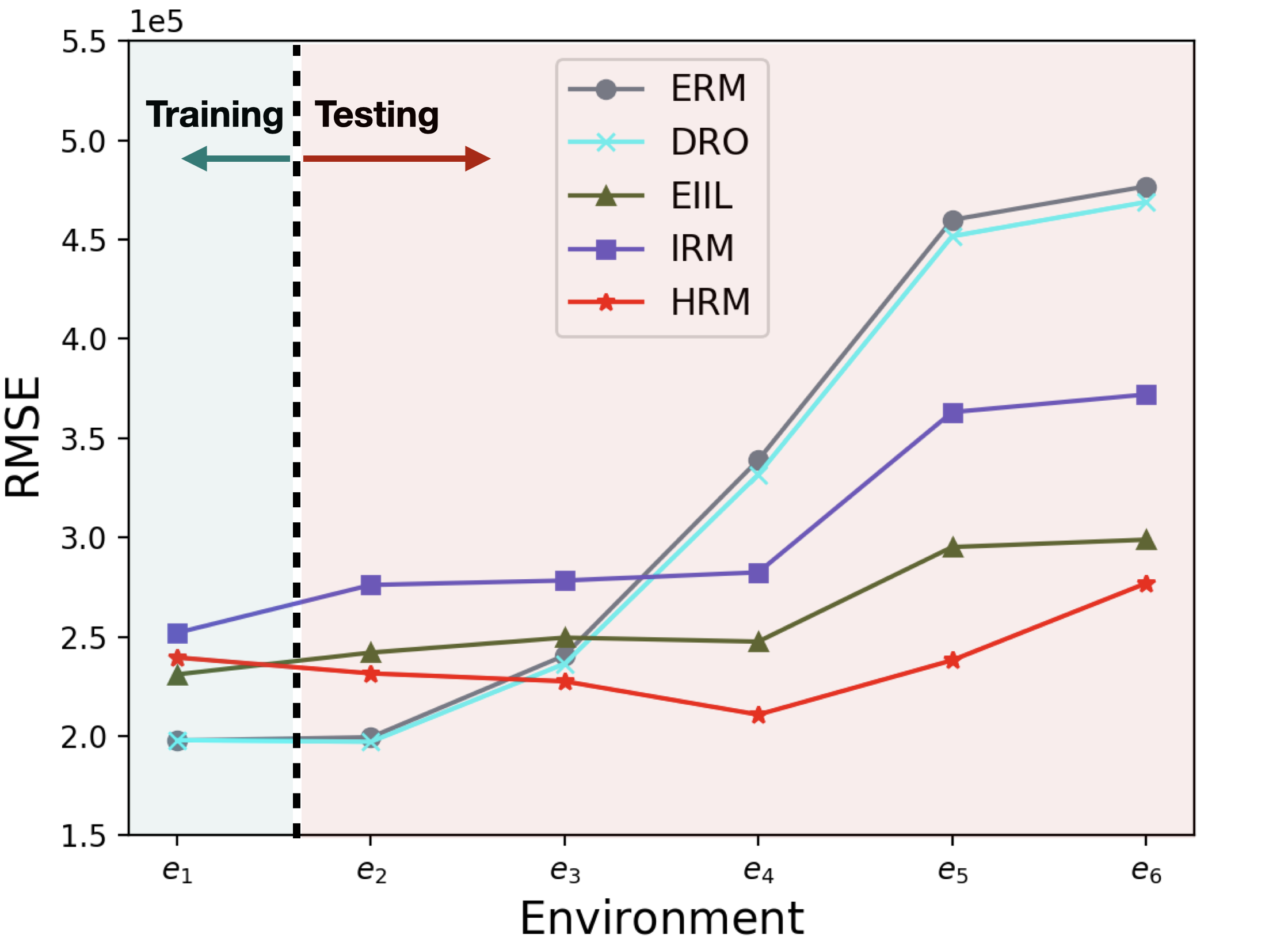}}

\caption{Results of real-world datasets, including training and testing performance for five methods.}
\vskip -0.1in
\end{figure*}

\subsection{Real-world Data}
We test our method on three real-world tasks, including car insurance prediction, people income prediction and house price prediction.
\subsubsection{Settings}
{\bf Car Insurance Prediction}\quad 
In this task, we use a real-world dataset for car insurance prediction (Kaggle). It is a classification task to predict whether a person will buy car insurance based on related information, such as vehicle damage, annual premium, vehicle age etc\footnote{https://www.kaggle.com/anmolkumar/health-insurance-cross-sell-prediction}.
We impose selection bias mechanism on the correlation between the outcome (i.e. the label indicating whether buying insurance) and the sex attribute to simulate multiple environments.
Specifically, we simulate different strengths $|r|$ of the spurious correlation between sex and target in training, and reverse the direction of such correlation in testing($+|r|$ in training and $-|r|$ in testing).
For IRM, in each setting, we divide the training data into three training environments with $r_1=0.95,r_2=0.9,r_3=-0.8$, and different overall correlation $r$ corresponds to different numbers of data in $e_1,e_2,e_3$.
We perform 5 experiments with varying $r$ and the results in both training and testing are shown in Figure \ref{img:car renting}.

{\bf People Income Prediction}\quad 
In this task we use the Adult dataset \cite{Dua:2019} to predict personal income levels as above or below \$50,000 per year based on personal details. We split the dataset into 10 environments according to demographic attributes $\mathrm{sex}$ and $\mathrm{race}$. 
In training phase, all methods are trained on pooled data including 693 points from environment 1 and 200 from environment 2, and validated on 100 sampled from both. 
For IRM, the ground-truth environment labels are provided.
In testing phase, we test all methods on the 10 environments and report the mis-classification rate on all environments in Figure \ref{img:income}.

{\bf House Price Prediction}\quad 
In this experiment, we use a real-world regression dataset (Kaggle) of house sales prices from King County, USA\footnote{ https://www.kaggle.com/c/house-prices-advanced-regression- techniques/data}. 
The target variable is the transaction price of the house and each sample contains 17 predictive variables such as the built year of the house, number of bedrooms, and square footage of home, etc. 
We simulate different environments according to the built year of the house, since it is fairly reasonable to assume the correlations among covariates and the target may vary along time. 
Specifically, we split the dataset into 6 periods, where each period approximately covers a time span of two decades. 
All methods are trained on data from the first period($[1900,1920)$) and test on the other periods.
For IRM, we further divide the training data into two environments where $built\ year \in [1900, 1910)$ and $[1910, 1920)$ respectively.
Results are shown in Figure \ref{img:house price}.

\subsubsection{Analysis}
From the results of three real-world tasks, we have the following observations and analysis:
{\bf ERM} achieves high accuracy in training while performing much worse in testing, indicating its inability in dealing with OOD predictions.
{\bf DRO}'s performance is not satisfactory, sometimes even worse than ERM. One plausible reason is its over-pessimistic nature which leads to too conservative predictors. 
Comparatively, invariant learning methods perform better in testing.
{\bf IRM} performs better than ERM and DRO, which shows the usefulness of environment labels for OOD generalization and the possibility of learning invariant predictor from multiple environments.
{\bf EIIL} performs inconsistently across different tasks, possibly due to its instability of the environment inference method.
In all tasks and almost all testing environments (16/18), {\bf HRM} consistently achieves the best performances. 
HRM even outperforms IRM significantly in a unfair setting where we provide perfect environment labels for IRM. 
One one side, it shows the limitation of manually labeled environments. 
On the other side, it demonstrates that, relieving the dependence on environment labels, HRM can effectively uncover and fully leverage the intrinsic heterogeneity in training data for invariant learning.

\section{Related Works}
There are mainly two branches of methods for the OOD generalization problem, namely distributionally robust optimization (DRO) \cite{DBLP:journals/mp/EsfahaniK18, duchi2018learning, SinhaCertifying,sagawa2019distributionally} and invariant learning \cite{arjovsky2019invariant, DBLP:journals/corr/abs-2008-01883, DBLP:rationalization, creager2020environment}.
DRO methods propose to optimize the worst-case risk within an uncertainty set, which lies around the observed training distribution and characterizes the potential testing distributions.
However, in real scenarios, to better capture the testing distribution, the uncertainty set should be pretty large, which also results in the over-pessimism problem of DRO methods\cite{does, frogner2019incorporating}.

Realizing the difficulty of solving OOD generalization problem without any prior knowledge or structural assumptions, invariant learning methods assume the existence of causally invariant relationships between some predictors $\Phi(X)$ and the target $Y$. 
\cite{arjovsky2019invariant} and \cite{DBLP:journals/corr/abs-2008-01883} propose to learning an invariant representation through multiple training environments.
\cite{DBLP:rationalization} also proposes to select features whose predictive relationship with the target stays invariant across environments.
However, their effectiveness relies on the quality of the given multiple training environments, and the role of environments remains vague theoretically.
Recently, \cite{creager2020environment} improves \cite{arjovsky2019invariant} by relaxing its requirements for multiple environments.
Specifically, \cite{creager2020environment} proposes a two-stage method, which firstly infers the environment division with a pre-provided biased model, and then performs invariant learning on the inferred environments.
However, the two stages cannot be jointly optimized, and the environment division relies on the given biased model and lacks theoretical guarantees.

\section{Discussions}
In this work, we theoretically analyze the role of environments in invariant learning, and propose our HRM for joint heterogeneity identification and invariant prediction, which relaxes the requirements for environment labels and opens a new direction for invariant learning.
To our knowledge, this is the first work to both theoretically and empirically analyze how the equality of multiple environments affects invariant learning.
This paper mainly focuses on the raw variable level with the assumption of $X=[\Phi^*, \Psi^*]^T$, which is able to cover a broad spectrum of applications, e.g. healthcare, finance, marketing etc, where the raw variables are informative enough.

However, our work has some limitations, which we hope to improve in the future.
Firstly, in order to achieve the mutual promotion, we should use the variant features $\Psi^*$ for heterogeneity identification rather than the invariant ones.
However, the process of invariant prediction continuously discards the variant features $\Psi^*$ (for invariant features or representation), which makes it quite hard to recover the variant features.
To overcome this, we focus on the simple setting where $X=[\Phi^*, \Psi^*]^T$, since we can directly obtain the variant features $\Psi^*$ when having invariant features $\Phi^*$.
To further extend the power of HRM, we will consider to incorporate representation learning from $X$ in future work.
Secondly, our clustering algorithm in $\mathcal{M}_c$ lacks theoretical guarantees for its convergence.
To the best of our knowledge, in order to theoretically analyze the convergence of a clustering algorithm, it is necessary to measure the distance between data points.
However, our clustering algorithm takes models' parameters as cluster centers and aims to cluster data points $(X,Y)$ according to their relationships between $X$ and $Y$, whose dissimilarity cannot be easily measured, since the relationship is statistical magnitude and cannot be calculated individually.
How to theoretically analyze the convergence property of such clustering algorithms remains unsolved.

\section{Acknowledgements}
This work was supported in part by National Key R\&D Program of China (No. 2018AAA0102004, No. 2020AAA0106300), National Natural Science Foundation of China (No. U1936219, 61521002, 61772304), Beijing Academy of Artificial Intelligence (BAAI), and a grant from the Institute for Guo Qiang, Tsinghua University.
Bo Li’s research was supported by the Tsinghua University Initiative Scientific Research Grant, No. 2019THZWJC11; Technology and Innovation Major Project of the Ministry of Science and Technology of China under Grant 2020AAA0108400 and 2020AAA01084020108403; Major Program of the National Social Science Foundation of China (21ZDA036).

\newpage
\bibliography{icml21}
\bibliographystyle{icml2021}

\appendix
\section{Additional Simulation Results and Details}

\textbf{Selection Bias}
In this setting, the correlations among covariates are perturbed through selection bias mechanism. 
According to assumption 2.1, we assume $X = [\Phi^*,\Psi^*]^T \in \mathbb{R}^d$ and $\Phi^* = [\Phi^*_1, \Phi^*_2, \dots, \Phi^*_{n_\phi}]^T \in \mathbb{R}^{n_\phi}$ is independent from $\Psi^* = [\Psi^*_1, \Psi^*_2, \dots, \Psi^*_{n_\psi}]\in \mathbb{R}^{n_\psi}$ while the covariates in $\Phi^*$ are dependent with each other. 
We assume $Y = f(\Phi^*) + \epsilon$ and $P(Y|\Phi^*)$ remains invariant across environments while $P(Y|\Psi^*)$ can arbitrarily change. 

Therefore, we generate training data points with the help of auxiliary variables $Z \in \mathbb{R}^{n_\phi+1}$ as following:
\begin{align}
Z_1, \dots, Z_{n_\phi+1} &\stackrel{iid}{\sim} \mathcal{N}(0,1.0) \\
\Psi^*_1, \dots, \Psi^*_{n_\psi} &\stackrel{iid}{\sim} \mathcal{N}(0,1.0) \\
\Phi^*_i = 0.8*Z_i + 0.2 * Z_{i+1} &\ \ \ \ \ for \ \ i = 1, \dots, n_\phi
\end{align}
To induce model misspecification, we generate $Y$ as:
\begin{equation}
Y = f(\Phi^*) + \epsilon = \theta_\phi*(\Phi^*)^T + \beta*\Phi^*_1\Phi^*_2\Phi^*_3+\epsilon
\end{equation}
where $\theta_\phi = [\frac{1}{2},-1, 1, -\frac{1}{2}, 1, -1 , \dots] \in \mathbb{R}^{n_\phi}$, and $\epsilon \sim \mathcal{N}(0, 0.3)$. 
As we assume that $P(Y|\Phi^*)$ remains unchanged while $P(Y|\Psi^*)$ can vary across environments, we design a data selection mechanism to induce this kind of distribution shifts.
For simplicity, we select data points according to a certain variable set $V_b \subset \Psi^*$:

\begin{align}
&\hat{P} = \Pi_{v_i \in V_b}|r|^{-5*|f(\phi) - sign(r)*v_i|}  \\
&\mu \sim Uni(0,1 ) \\
&M(r;(x,y)) =
\begin{cases}
1, \ \ \ \ \ &\text{$\mu \leq \hat{P}$ } \\
0, \ \ \ \ \ &\text{otherwise}
\end{cases} 
\end{align}  
where $|r| > 1$ and $V_b \in \mathbb{R}^{n_b}$.
Given a certain $r$, a data point $(x,y)$ is selected if and only if $M(r;(x,y))=1$ (i.e. if $r>0$, a data point whose $V_b$ is close to its $Y$ is more probably to be selected.)

Intuitively, $r$ eventually controls the strengths and direction of the spurious correlation between $V_b$ and $Y$(i.e. if $r>0$, a data point whose $V_b$ is close to its $Y$ is more probably to be selected.).
The larger value of $|r|$ means the stronger spurious correlation between $V_b$ and $Y$, and $r \ge 0$ means positive correlation and vice versa. 
Therefore, here we use $r$ to define different environments.

In training, we generate $sum$ data points, where $\kappa\cdot sum$ points from environment $e_1$ with a predefined $r$ and $(1-\kappa)sum$ points from $e_2$ with $r=-1.1$. 
In testing, we generate data points for 10 environments with $r \in [-3,-2,-1.7,\dots,1.7,2,3]$. 
$\beta$ is set to 1.0. 

Apart from the two scenarios in main body, we also conduct scenario 3 and 4 with varying $\kappa, n$ and $n_b$ respectively.

\begin{table*}[t]
	\centering
	\caption{Results in selection bias simulation experiments of different methods with varying sample size $sum$, ratio $\kappa$ and variant dimensions $n_b$ of training data, and each result is averaged over ten times runs.}
	\label{tab:sim_selection}
	
	\resizebox{0.95\textwidth}{28mm}{
	\begin{tabular}{|l|c|c|c|c|c|c|c|c|c|}
		\hline
 		\multicolumn{10}{|c|}{\textbf{Scenario 3: varying ratio $\kappa$ and sample size $sum$\quad($d=10,r = 1.9, n_b=1$)}}\\
 		\hline
 		$\kappa,n$&\multicolumn{3}{|c|}{$\kappa=0.90, sum=1000$}&\multicolumn{3}{|c|}{$\kappa=0.95, sum=2000$}&\multicolumn{3}{|c|}{$\kappa=0.975, sum=4000$}\\
 		\hline
		Methods &  $\mathrm{Mean\_Error}$ & $\mathrm{Std\_Error}$& $\mathrm{Max\_Error}$ &$\mathrm{Mean\_Error}$ & $\mathrm{Std\_Error}$& $\mathrm{Max\_Error}$ &  $\mathrm{Mean\_Error}$ & $\mathrm{Std\_Error}$ & $\mathrm{Max\_Error}$ \\
 		\hline %
 		ERM & 0.477 & 0.061 & 0.530& 0.510 & 0.108& 0.608  & 0.547 & 0.150 & 0.687\\ 
 		DRO & 0.480 & 0.107 &0.597 &  0.512 & 0.111 & 0.625&0.608 & 0.227 & 0.838\\
 		EIIL & 0.476 & 0.063& 0.529 & 0.507 & 0.102 & 0.613& 0.539	& 0.148 & 0.689\\
 		\hline
 		IRM(with $\mathcal{E}_{tr}$ label) & 0.455 & 0.015 &0.471& 0.456 & 0.015 & 0.472&0.456 & 0.015& 0.472\\
 		\hline
 		HRM &\bf 0.450 & \bf 0.010 & \bf 0.461&\bf 0.447 & \bf 0.011 &\bf 0.465 &\bf 0.447 & \bf 0.010 & \bf 0.463\\
 		\hline

 		\multicolumn{10}{|c|}{\textbf{Scenario 4: varying variant dimension $n_{b}$\quad($d = 10, sum=2000,\kappa=0.95, r = 1.9,n_b=1$)}}\\
 		\hline
 		$n_b$&\multicolumn{3}{|c|}{$n_{b}=1$}&\multicolumn{3}{|c|}{$n_{b}=3$}&\multicolumn{3}{|c|}{$n_{b}=5$}\\
 		\hline
		Methods &  $\mathrm{Mean\_Error}$ & $\mathrm{Std\_Error}$& $\mathrm{Max\_Error}$ &$\mathrm{Mean\_Error}$ & $\mathrm{Std\_Error}$& $\mathrm{Max\_Error}$ &  $\mathrm{Mean\_Error}$ & $\mathrm{Std\_Error}$ & $\mathrm{Max\_Error}$ \\ 		
		\hline %
 		ERM & 0.510 & 0.108 & 0.608&0.468 & 0.110  &0.583& 0.445 & 0.112 & 0.567\\ 
 		DRO & 0.512 & 0.111 & 0.625 & 0.515 & 0.107 & 0.617& 0.454 & 0.122 & 0.577\\
  		EIIL & 0.520 & 0.111 & 0.613& 0.469 & 0.111 & 0.581& 0.454	& 0.100&0.557\\
  		\hline
 		IRM(with $\mathcal{E}_{tr}$ label) & 0.456 & 0.015 & 0.472 &0.432 & 0.014 & 0.446& 0.414 & 0.061& 0.475\\
 		\hline
 		HRM &\bf 0.447 & \bf 0.011 & \bf 0.465 & \bf 0.413 & \bf 0.012 & \bf 0.431& \bf 0.402 & \bf 0.057 & \bf 0.462\\
 		\hline
	\end{tabular}
	}
\end{table*}

\textbf{Anti-Causal Effect}
Inspired by \cite{arjovsky2019invariant}, in this setting, we introduce the spurious correlation by using anti-causal relationship from the target $Y$ to the variant covariates $\Psi^*$.

We assume $X=[\Phi^*,\Psi^*]^T \in \mathbb{R}^d$ and $\Phi^* = [\Phi^*_1, \Phi^*_2, \dots, \Phi^*_{n_\phi}]^T \in \mathbb{R}^{n_\phi}$, $\Psi^* = [\Psi^*_1, \Psi^*_2, \dots, \Psi^*_{n_\psi}]\in \mathbb{R}^{n_\psi}$
Data Generation process is as following:
\begin{align}
	\Phi^* &\sim \sum_{i=1}^k z_k \mathcal{N}(\mu_i,I)\\
	Y &= \theta_\phi^T\Phi^* + \beta \Phi^*_1\Phi^*_2\Phi^*_3+\mathcal{N}(0,0.3)\\
	\Psi^* &= \theta_\psi Y + \mathcal{N}(0,\sigma(\mu_i)^2)
\end{align}  
where $\sum_{i=1}^k z_i = 1\ \&\  z_i >= 0$ is the mixture weight of $k$ Gaussian components, $\sigma(\mu_i)$ means the Gaussian noise added to $\Psi^*$ depends on which component the invariant covariates $\Phi^*$ belong to and $\theta_\psi \in \mathbb{R}^{n_\psi}$. 
Intuitively, in different Gaussian components, the corresponding correlations between $\Psi^*$ and $Y$ are varying due to the different value of $\sigma(\mu_i)$. 
The larger the $\sigma(\mu_i)$ is, the weaker correlation between $\Psi^*$ and $Y$. 
We use the mixture weight $Z=[z_1,\dots,z_k]^T$ to define different environments, where different mixture weights represent different overall strength of the effect $Y$ on $\Psi^*$.
In this experiment, we set $\beta=0.1$ and build 10 environments with varying $\sigma$ and the dimension of $\Phi^*,\Psi^*$, the first three for training and the last seven for testing. 
Specifically, we set $\beta=0.1$, $\mu_1=[0,0,0,1,1]^T,\mu_2 = [0,0,0,1,-1]^T,\mu_2=[0,0,0,-1,1]^T,\mu_4=\mu_5=\dots=\mu_{10}=[0,0,0,-1,-1]^T$, $\sigma(\mu_1)=0.2, \sigma(\mu_2)=0.5,\sigma(\mu_3)=1.0$ and $[\sigma(\mu_4), \sigma(\mu_5),\dots,\sigma(\mu_{10})]=[3.0,5.0,\dots,15.0]$. 
$\theta_\phi, \theta_\psi$ are randomly sampled from $\mathcal{N}(1,I)$ and $\mathcal{N}(0.5,0.1I)$ respectively.
We run experiments for 10 times and average the results.

\section{Proofs}

\subsection{Proof of Theorem 2.1}
First, we would like to prove that a random variable satisfying assumption 2.1 is MIP.
 \begin{theorem}
 	\label{theorem:equ}
     A representation $\Phi^* \in \mathcal{I}$ satisfying assumption 2.1 is the maximal invariant predictor.   
 \end{theorem}
	
 \begin{proof}
 	$\rightarrow$: To prove $\Phi^* = \arg\min_{Z\in \mathcal{I}}I(Y;Z)$.
 	If $\Phi^*$ is not the maximal invariant predictor, assume $\Phi' = \arg\max_{Z \in \mathcal{I}} I(Y;Z)$. 
 	Using functional representation lemma, consider $(\Phi^*, \Phi')$, there exists random variable $\Phi_{extra}$ such that 
 	$\Phi^{'} = \sigma(\Phi^*, \Phi_{extra})$ and $\Phi^{*}\perp \Phi_{extra}$. 
 	Then $I(Y;\Phi^{'}) = I(Y;\Phi^{*}, \Phi_{extra}) = I(f(\Phi^*);\Phi^*, \Phi_{extra}) = I(f(\Phi^*);\Phi^*)$.

 	$\leftarrow$: To prove the maximal invariant predictor $\Phi^*$ satisfies the sufficiency property in assumption 2.1.
	
 	The converse-negative proposition is :
 	\begin{equation}
 		 Y \neq f(\Phi^*)+\epsilon \rightarrow \Phi^* \neq \arg\max_{Z\in\mathcal{I}}I(Y;Z)
 	\end{equation}
 	Suppose $Y \neq f(\Phi^*)+\epsilon$ and $\Phi^* = \arg\max_{Z\in\mathcal{I}}I(Y;Z)$, and suppose $Y = f(\Phi^{'})+\epsilon$ where $\Phi^{'}\neq \Phi^*$.
 	Then we have:
 	\begin{equation}
 		I(f(\Phi^{'});\Phi^*) \leq I(f(\Phi^{'});\Phi^{'}) 
 	\end{equation}
 	Therefore, $\Phi^{'}=\arg\max_{Z\in\mathcal{I}}I(Y;Z)$
 \end{proof}

Then we provide the proof of theorem 2.1.

\begin{theorem}
 	Let $g$ be a strictly convex, differentiable function and let $D$ be the corresponding Bregman Loss function. Let $\Phi^*$ is the maximal invariant predictor with respect to $I_{\mathcal{E}}$, and put $h^*(X)=\mathbb{E}_Y[Y|\Phi^*]$. Under assumption 2.2, we have:
 	\begin{equation}
 		h^* = \arg\min_h \sup_{e \in \mathrm{supp}(\mathcal{E})} \mathbb{E}[D(h(X),Y)|e]
 	\end{equation}
\end{theorem}
 \begin{proof}
 	Firstly, according to theorem \ref{theorem:equ}, $\Phi^*$ satisfies assumption 2.1.	
 	Consider any function $h$, we would like to prove that for each distribution $P^e$($e \in \mathcal{E}$), there exists an environment $e'$ such that:
 	\begin{equation}
 		\mathbb{E}[D(h(X),Y)|e']\geq \mathbb{E}[D(h^*(X),Y)|e]
 	\end{equation}
 	For each $e\in\mathcal{E}$ with density $([\Phi,\Psi],Y)\mapsto P(\Phi,\Psi,Y)$, we construct environment $e'$ with density $Q(\Phi,\Psi,Y)$ that satisfies: (omit the superscript $*$ of $\Phi$ and $\Psi$ for simplicity)
 	\begin{equation}
 		Q(\Phi,\Psi,Y)=P(\Phi,Y)Q(\Psi)
 	\end{equation}
 	Note that such environment $e'$ exists because of the heterogeneity  property assumed in assumption 2.2.
 	Then we have:
 	\begin{align}
 		&\int D(h(\phi,\psi),y)q(\phi,\psi,y)d\phi d\psi dy\\
 		& = \int_\psi \int_{\phi,y} D(h(\phi,\psi),y)p(\phi,y)q(\psi)d\phi dy d\psi \\
 		& = \int_\psi \int_{\phi,y}D(h(\phi,\psi),y)p(\phi,y)d\phi dy q(\psi)d\psi\\
 		& \geq \int_\psi \int_{\phi,y} D(h^*(\phi,\psi),y)p(\phi,y)d\phi dy q(\psi)d\psi\\
 		&= \int_\psi \int_{\phi,y} D(h^*(\phi),y)p(\phi,y)d\phi dy q(\psi)d\psi\\
 		&= \int_{\phi,y} D(h^*(\phi),yp(\phi,y)d\phi dy\\
 		&= \int_{\phi,\psi,y} D(h^*(\phi),y)p(\phi,\psi,y)d\phi d\psi dy\\
 	\end{align}
 \end{proof}

\subsection{Proof of Theorem 2.2}
\begin{theorem}
    $\mathcal{I}_{\mathcal{E}} \subseteq \mathcal{I}_{\mathcal{E}_{tr}}$
\end{theorem}
 \begin{proof}
 	Since $\mathcal{E}_{tr}\subseteq \mathcal{E}$, then for any $S \in \mathcal{I}_{\mathcal{E}}$, $S \in \mathcal{I}_{\mathcal{E}_{tr}}$.
 \end{proof}

\subsection{Proof of Theorem 2.3}
\begin{theorem}
    \label{theorem:useless env}
    Given set of environments $\mathrm{supp}(\hat{\mathcal{E})}$, denote the corresponding invariance set $\mathcal{I}_{\hat{\mathcal{E}}}$ and the corresponding maximal invariant predictor $\hat{\Phi}$. For one newly-added environment $e_{new}$ with distribution $P^{new}(X,Y)$, if $P^{new}(Y|\hat{\Phi}) = P^e(Y|\hat{\Phi})$ for $e \in \mathrm{supp}(\hat{\mathcal{E}})$, the invariance set constrained by $\mathrm{supp}(\hat{\mathcal{E}})\cup \{e_{new}\}$ is equal to $\mathcal{I}_{\hat{\mathcal{E}}}$.
\end{theorem}

\begin{proof}
	Denote the invariance set with respect to $\mathrm{supp}(\hat{\mathcal{E}}\cup \{e_{new}\})$ as $\mathcal{I}_{new}$, it is easy to prove that $\forall S \in \mathcal{I}_{\hat{\mathcal{E}}}$, we have $S\in\mathcal{I}_{new}$, since the newly-added environment cannot exclude any variables from the original invariance set.
\end{proof}

\subsection{Proof of Theorem 4.1}
\begin{theorem}
	\label{theorem: mp}
    Given $\mathcal{E}_{tr}$, the learned $\Phi(X)=M\odot X$ is the maximal invariant predictor of $\mathcal{I}_{\mathcal{E}_{tr}}$.
\end{theorem}
\begin{proof}
	The objective function for $\mathcal{M}_p$ is 
	\begin{equation}
		\mathcal{L}_p(M\odot X,Y;\theta) = \mathbb{E}_{\mathcal{E}_{tr}}[\mathcal{L}^e] + \lambda \mathrm{trace}(\mathrm{Var}_{\mathcal{E}_{tr}}(\nabla_\theta\mathcal{L}^e))
	\end{equation}
	Here we prove that the minimum of objective function can be achieved when $\Phi(X)=M\odot X$ is the maximal invariant predictor.
	According to theorem \ref{theorem:equ}, $\Phi(X)$ satisfies assumption 2.1, which indicates that $P^e(Y|\Phi(X))$ stays invariant.
	
	From the proof in C.2 in \cite{DBLP:journals/corr/abs-2008-01883}, $I(Y;\mathcal{E}|\Phi(X))=0$ indicates that $\mathrm{trace}(\mathrm{Var}_{\mathcal{E}_{tr}}(\nabla_\theta\mathcal{L}^e))=0$.
	
	Further, from the sufficiency property, the minimum of $\mathcal{L}^e$ is achieved with $\Phi(X)$.
	Therefore, $\mathbb{E}_{\mathcal{E}_{tr}}[\mathcal{L}^e] + \lambda \mathrm{trace}(\mathrm{Var}_{\mathcal{E}_{tr}}(\nabla_\theta\mathcal{L}^e))$ reaches the minimum with $\Phi(X)$ being the MIP.($\lambda \geq 0$)
\end{proof}

\subsection{Proof of Theorem 4.2}
\begin{theorem}
    \label{theorem:kl}
    For $e_i, e_j \in \mathrm{supp}(\mathcal{E}_{tr})$, assume that $X = [\Phi^*, \Psi^*]^T$ satisfying Assumption 2.1, where $\Phi^*$ is invariant and $\Psi^*$ variant.
    Then under Assumption 4.1, we have $
        \mathrm{D_{KL}}(P^{e_i}(Y|X)\|P^{e_j}(Y|X)) \leq \mathrm{D_{KL}}(P^{e_i}(Y|\Psi^*)\|P^{e_j}(Y|\Psi^*))$
\end{theorem}

\begin{proof}
	\begin{align}
		&\ \ \ \ \ \mathrm{D_{KL}}(P^{e_i}(Y|X)\|P^{e_j}(Y|X))\\
		&= \mathrm{D_{KL}}(P^{e_i}(Y|\Phi^*, \Psi^*)\|P^{e_j}(Y|\Phi^*,\Psi^*))\\
		&= \int\int\int p_i(y,\phi,\psi)\log\left[\frac{p_i(y|\phi,\psi)}{p_j(y|\phi,\psi)}\right]dyd\phi d\psi
	\end{align}
	
	Therefore, we have
	\begin{small}
	\begin{align}
		&\mathrm{D_{KL}}(P^{e_i}(Y|\Psi)\|P^{e_j}(Y|\Psi)) - \mathrm{D_{KL}}(P^{e_i}(Y|X)\|P^{e_j}(Y|X)) \\
		&= \int\int\int p_i(y,\phi,\psi)\left(\log\frac{p_i(y|\psi)}{p_j(y|\psi)}-\log\frac{p_i(y|\phi,\psi)}{p_j(y|\phi,\psi)}\right)dyd\phi d\psi\\
		&= \int\int\int p_i(y,\phi,\psi)\left(\log\frac{p_i(y|\psi)}{p_i(y|\phi,\psi)} - \log\frac{p_j(y|\psi)}{p_j(y|\phi,\psi)}    \right)dyd\phi d\psi\\
		&= I_{i,j}^c(Y;\Phi^*|\Psi^*) - I_{i}(Y;\Phi^*|\Psi^*)
	\end{align}
	\end{small}
	Therefore, we have
	\begin{small}
	\begin{equation}
		\mathrm{D_{KL}}(P^{e_i}(Y|X)\|P^{e_j}(Y|X)) \leq \mathrm{D_{KL}}(P^{e_i}(Y|\Psi^*)\|P^{e_j}(Y|\Psi^*))
	\end{equation}
	\end{small}
\end{proof}

\subsection{Proof of Theorem 4.3}
\begin{theorem}
    Under Assumption 2.1 and 2.2, for the proposed $\mathcal{M}_c$ and $\mathcal{M}_p$, we have the following conclusions:
    1. Given environments $\mathcal{E}_{tr}$ such that $\mathcal{I}_{\mathcal{E}}=\mathcal{I}_{\mathcal{E}_{tr}}$, the learned $\Phi(X)$ by $\mathcal{M}_p$ is the maximal invariant predictor of $\mathcal{I}_{\mathcal{E}}$.
    2. Given the maximal invariant predictor $\Phi^*$ of $\mathcal{I}_{\mathcal{E}}$, assume the pooled training data is made up of data from all environments in $\mathrm{supp}(\mathcal{E})$, then the invariance set $\mathcal{I}_{\mathcal{E}_{tr}}$ regularized by learned environments $\mathcal{E}_{tr}$ is equal to $\mathcal{I}_{\mathcal{E}}$.
\end{theorem}
\begin{proof}
	For 1, according to theorem \ref{theorem: mp}, the learned $\Phi(X)$ by $\mathcal{M}_p$ is the maximal invariant predictor of $\mathcal{I}_{\mathcal{E}_{tr}}$.
	Therefore, if $\mathcal{I}_{\mathcal{E}}=\mathcal{I}_{\mathcal{E}_{tr}}$, $\Phi(X)$ is the real maximal invariant predictor.
	
	For 2, assume that $P_{train}(X,Y)=\sum_{e\in\mathcal{E}} w_eP^e(X,Y)$, we would like to prove that $\mathrm{D}_{KL}(P_{train}(Y|\Psi^*)\|Q)$ reaches minimum when the components in the mixture distribution $Q$ corresponds to distributions for $e\in\mathcal{E}$.
	Since the learned $\Phi(X)$ by $\mathcal{M}_p$ is the maximal invariant predictor of $\mathcal{I}_{\mathcal{E}}$, the corresponding $\Psi(X)$ is exactly the $\Psi^*(X)$.
	Then taking $Q^*=\sum_{e\in\mathcal{E}}w_eP^e(Y|\Psi^*)$, we have $\forall Q \in \mathcal{Q}$, 
	\begin{equation}
		\mathrm{D_{KL}}(P_{train}(Y|\Psi^*)\|Q^*)\leq \mathrm{D_{KL}}(P_{train}(Y|\Psi^*)\|Q)
	\end{equation}
	Therefore, the components in $Q^*$ correspond to $P^e$ for $e\in\mathcal{E}$, which makes $\mathcal{I}_{\mathcal{E}_{tr}}$ approaches to $\mathcal{I}_{\mathcal{E}}$.
\end{proof}


\end{document}